\documentclass[11pt]{article}

\usepackage[letterpaper,margin=1in]{geometry}
\usepackage[T1]{fontenc}
\usepackage[utf8]{inputenc}
\usepackage{times}
\usepackage{amssymb}
\usepackage{microtype}
\usepackage{amsmath,amsthm}
\usepackage{graphicx}
\usepackage{booktabs}
\usepackage{float}
\usepackage{tabularx}
\usepackage{array}
\usepackage{multirow}
\usepackage{xcolor}
\usepackage{natbib}
\usepackage[colorlinks=true,linkcolor={blue!60!black},citecolor={blue!60!black},urlcolor={blue!60!black}]{hyperref}
\usepackage[font=small,labelfont=bf]{caption}
\usepackage{titlesec}

\setcitestyle{authoryear,round,aysep={},yysep={,},notesep={, }}
\titlespacing*{\paragraph}{0pt}{6pt plus 2pt minus 1pt}{0.6em}

\theoremstyle{plain}
\newtheorem{theorem}{Theorem}
\newtheorem{proposition}{Proposition}

\newcommand{\E}{\mathbb{E}}

\title{\textbf{ProbeAudit: Evaluating the Reliability of\\Experiment Selection in AI Scientists}}
\author{Yifan Guo\\[2pt] Stony Brook University}
\date{}

\hypersetup{pdftitle={ProbeAudit: Evaluating the Reliability of Experiment Selection in AI Scientists},pdfauthor={Yifan Guo}}

\begin{document}
\maketitle

\begin{abstract}
Benchmarks for AI scientists usually grade end-to-end outcomes. This conflates three questions: whether an agent chose a good experiment, whether its estimate of each experiment's value was trustworthy, and whether its final conclusion was right. We present ProbeAudit, a stage-level instrument that audits experiment-selection decisions against a declared utility. Matched-seed pilot runs and simultaneous confidence bounds label each proposed probe as \emph{supported}, \emph{contradicted} or \emph{unresolved}; an anytime-valid race covers the policy's own choices, and a covariance-aware (Hotelling) cross-probe audit covers externally proposed actions. On 167,328 stored training-time fits across eight libraries, focused allocation at equal fit budgets certifies 48.1\% of decisions within 10\% score regret, against 38.1\% for uniform allocation, with no violation against two seed-disjoint references. On a new Ridge library with 13,824 fits, the covariance-aware audit refutes 121 of 960 negative-control actions and reveals that Bonferroni multiplicity, more than pilot noise, limits power. Stress tests with heteroscedastic noise and calibration drift show that certified rankings can coexist with poor diagnoses. We distill the evidence into a minimum trace-reporting standard for experiment-selection modules of AI scientists.
\end{abstract}

\section{Introduction}
Autonomous scientific systems generate candidate explanations, run experiments, update beliefs and report conclusions. Suppose a model improves after feature normalization. The improvement might reflect conditioning, regularization, sample size or another training defect, and reproducing it cannot tell these apart. A diagnostic intervention is useful if competing mechanisms predict different outcomes for it. Yet each mechanism's prediction is often estimated from a few costly runs, so the apparently most informative intervention may win only because of seed noise.

This is a measurement problem for evaluations of research agents. A benchmark may credit an agent for picking an experiment or for naming a planted cause, but these answer different questions. Is the chosen experiment best under the stated utility? Is that utility appropriate for diagnosis? Is the final claim correct? How much computation did it take? A system can pass one test and fail another, and grading only the final output hides unreliable intermediate decisions.

We isolate the experiment-selection stage. A researcher or agent supplies a finite set of candidate pipelines and interventions; a policy chooses interventions, observes target responses and updates its belief. Our utility is posterior-weighted variance of expected responses, a classical model-discrimination criterion that coincides with Box--Hill under common Gaussian noise \citep{box1967}. We do not propose a new acquisition rule. We make the reliability of its \emph{estimated} ranking auditable, so that an evaluator can distinguish lack of evidence from a bad choice and can price the cost of resolving it.

\paragraph{Contributions.}
(1) A stage-level protocol that separates ranking reliability, diagnostic validity, abstention and resource cost. (2) A simultaneous matched-seed certificate with a relative-regret tolerance, a distribution-free median variant, and a three-way action audit with rectangular and covariance-aware cross-probe contrasts (Table~\ref{tab:families}). (3) A three-stage empirical study: equal-budget certification on eight stored libraries, sequential diagnosis under stress, and an external-policy trace audit on a newly executed regression library, including a power decomposition and an error-spending ablation. (4) A minimum report for experiment-selection traces (Table~\ref{tab:report}).

\paragraph{Scope and non-claims.}
The policies audited here are scripted (MDV, EIG, random and a negative control); we evaluate no LLM agent, and token cost is therefore zero. The audit consumes only a policy's belief and proposed action at each step, so an LLM policy plugs in without change (Section~\ref{sec:implications}). Candidate causes are planted pipeline defects with a known answer, not discovered mechanisms. Our certificates concern a declared score, not causal truth.

\section{Evaluation target and related work}
Let $\mathcal{H}=\{h_1,\dots,h_K\}$ be a prespecified set of executable candidate configurations and $\mathcal{E}=\{e_1,\dots,e_E\}$ a set of interventions (probes). The target is a hidden configuration $h^\star$. A policy receives a prior $p_0$, may use pilot runs of all candidates, and at step $t$ chooses $e_t$ and observes only that probe's target outcome $y_t$. The grader knows $h^\star$ and holds seed-disjoint reference responses that the policy never sees.

We separate three evaluation levels. \emph{Ranking reliability}: does the chosen probe maximize, or nearly maximize, the declared population score? \emph{Diagnostic validity}: does the sequence identify $h^\star$ without confident wrong conclusions? \emph{Resource efficiency}: how many pilot fits and target interventions are consumed? The second can fail when the first holds, because the score or library may be misspecified.

Model discrimination has a long history in experimental design \citep{hunter1965,box1967,schwaab2008}; Bayesian expected information gain (EIG) instead chooses by expected entropy reduction \citep{lindley1956,chaloner1995,rainforth2024}. Common random numbers and paired ranking-and-selection reduce comparison noise \citep{glasserman1992,kim2001}. Our race resembles LUCB \citep{kalyanakrishnan2012} with conservative error spending rather than optimal sample complexity \citep{garivier2016,howard2021}. What is specific here is a posterior-indexed quadratic score whose single calibration event can be reused as beliefs change. Research-agent benchmarks such as MLAgentBench, ScienceAgentBench and MLE-bench grade agents on full tasks \citep{huang2024,chen2025,chan2025}; ProbeAudit complements them as a trace-level audit of one decision stage.

\section{The evaluation instrument}

\begin{figure}[!t]
\centering
\includegraphics[width=\textwidth]{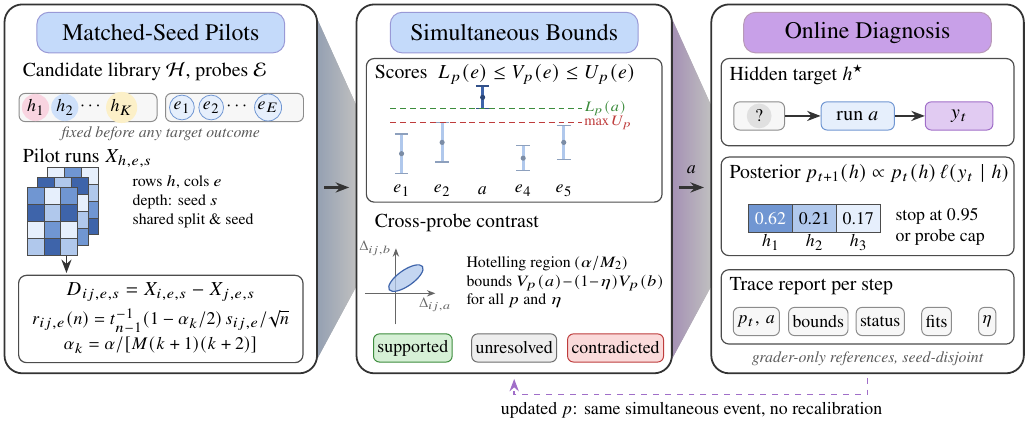}
\caption{Overview of ProbeAudit. Matched-seed pilots give paired differences and Student radii; simultaneous bounds certify, refute or leave unresolved each proposed probe, through score intervals or Hotelling contrast regions; the executed probe reveals only its own target outcome. Dashed arrow: the updated belief returns to the audit, which reuses the same coverage event without recalibration. The certificate concerns the declared score, not the truth of the candidate library.}
\label{fig:overview}
\end{figure}

\subsection{A classical score with noisy inputs}
For candidate $h$ and probe $e$ let $m_{h,e}$ be the mean response. For a probability vector $p$ on $\mathcal{H}$, let $\overline{m}_p(e)=\sum_h p_h m_{h,e}$ and
\begin{align}
V_p(e) &= \sum_h p_h\bigl(m_{h,e}-\overline{m}_p(e)\bigr)^2 \label{eq:mdv}\\
&= \sum_{i<j} p_ip_j\Delta_{ij,e}^2, \qquad \Delta_{ij,e}=m_{i,e}-m_{j,e}. \label{eq:pair}
\end{align}
The pairwise form shows which candidate pairs drive this mechanism-disagreement variance (MDV). With common observation variance $\sigma^2$, $V_p(e)/\sigma^2$ equals the posterior-weighted sum of Gaussian KL divergences over ordered candidate pairs, so MDV and Box--Hill share a maximizer; with mechanism-dependent variances they need not. EIG is locally proportional to $V_p(e)$ at low signal \citep{guo2005} and can rank differently elsewhere. A plug-in policy estimates $m_{h,e}$ from few pilots, so a large estimated score may reflect either a real gap or noise. We therefore audit the ranking under the declared utility without presuming that the utility is right for diagnosis.

\subsection{Matched-seed calibration}
Run all $K$ candidates on shared split and seed $s$ for probe $e$, giving $X_{h,e,s}$ and paired differences $D_{ij,e,s}=X_{i,e,s}-X_{j,e,s}$. Within a pair, assume these are i.i.d.\ Gaussian across seeds with mean $\Delta_{ij,e}$; arbitrary dependence across pairs and probes is allowed. Let $M=E\binom{K}{2}$ and inspect at $n=n_0+kb$, $k\ge 0$, $n_0\ge 2$, with Student radius
\begin{equation}
r_{ij,e}(n)=\frac{s_{ij,e}(n)}{\sqrt{n}}\,t_{n-1}^{-1}\!\left(1-\frac{\alpha}{2M(k+1)(k+2)}\right),
\label{eq:radius}
\end{equation}
where $s_{ij,e}(n)$ is the sample standard deviation of the differences. For inspected sizes $n_e$, define
\begin{align}
L_p(e) &= \textstyle\sum_{i<j} p_ip_j\bigl[|\widehat\Delta_{ij,e}|-r_{ij,e}\bigr]_+^2, \label{eq:L}\\
U_p(e) &= \textstyle\sum_{i<j} p_ip_j\bigl[|\widehat\Delta_{ij,e}|+r_{ij,e}\bigr]^2. \label{eq:U}
\end{align}
Pairing cancels shared seed noise: with marginal SD $\sigma$ and within-seed correlation $\rho$, a difference has SD $\sigma\sqrt{2(1-\rho)}$, so relative to two added cellwise intervals the paired radius shrinks by $\sqrt{(1-\rho)/2}$ times a ratio of Student quantiles.

\begin{theorem}[Simultaneous ranking audit]\label{thm:ranking}
Fix $K,E\ge 2$ and $0<\alpha<1$, and let each pair--probe seed stream be revealed in prefix order. With probability at least $1-\alpha$, $L_p(e)\le V_p(e)\le U_p(e)$ holds simultaneously for all probes, inspected sizes and probability vectors $p$. On this event, at any data-dependent inspection time and for any $\eta\in[0,1)$,
\begin{equation}
L_p(a) > (1-\eta)\max_{e\neq a} U_p(e)
\label{eq:cert}
\end{equation}
implies $V_p(a)\ge(1-\eta)\max_e V_p(e)$; for $\eta=0$ the maximizer is unique.
\end{theorem}

\begin{proof}
The Student pivot errs with probability at most $\alpha/[M(k+1)(k+2)]$ at look $k$; a union bound over the $M$ pair--probe streams and $k\ge 0$ totals $\alpha$ because $\sum_{k\ge 0}1/[(k+1)(k+2)]=1$. On the complement, each $|\Delta_{ij,e}|$ lies between the absolute endpoints of its interval; squaring and summing with nonnegative weights bounds $V_p$ for every $p$ with no union over posteriors. Then $V_p(a)\ge L_p(a)>(1-\eta)U_p(e)\ge(1-\eta)V_p(e)$ for every rival, and trivially against $a$. The event is shared by all looks and data-dependent $p,\eta$, so the probability of ever issuing a false certificate in this fixed library is at most $\alpha$. This is an unconditional familywise guarantee, not a bound conditional on certification.
\end{proof}

The tolerance $\eta$ should be declared in advance or its full sensitivity curve reported; a 10\% score tolerance is not a 10\% tolerance on diagnostic error.

\subsection{An auditable selection policy}
Given $p$ and a calibration state, let $a=\arg\max_e L_p(e)$ and $c=\arg\max_{e\neq a}U_p(e)$. If \eqref{eq:cert} holds, report $a$, its bounds, the tolerance and the fit count; otherwise add a matched batch to the less-sampled of $a$ and $c$ and repeat. At the sample cap return \textsc{Abstain} with the intervals. An application that must act may use a separately labeled plug-in fallback, which carries no certificate; traces must keep this distinction. Cellwise (unpaired) and uniform-allocation comparators reuse the same bounds and isolate the effects of pairing and focusing.

\paragraph{Distribution-free variant.}
When differences are i.i.d.\ but not Gaussian, order statistics $[D_{(j)},D_{(n-j+1)}]$ cover the pairwise median $\nu_{ij,e}$, with $j$ the largest integer satisfying $2\Pr\{\mathrm{Bin}(n,\tfrac12)\le j-1\}\le\alpha/[M(k+1)(k+2)]$ (unbounded if none exists). The same argument certifies $V_p^{\mathrm{med}}(e)=\sum_{i<j}p_ip_j\nu_{ij,e}^2$, a different estimand from $V_p$ unless means equal medians (Appendix~\ref{app:median}).

\subsection{Auditing an externally proposed action}
For a probe $a$ proposed by any policy, report \emph{supported} if \eqref{eq:cert} holds, \emph{contradicted} if $U_p(a)<(1-\eta)\max_{e\neq a}L_p(e)$, and \emph{unresolved} otherwise. On the coverage event, support implies regret at most $\eta$ and contradiction implies regret above $\eta$; unresolved is not an error. A policy need not use MDV to be audited against MDV, and the audit is no verdict on the policy's own utility.

Separate score intervals ignore that two probes are evaluated on the same seeds. At a common fixed pilot size, the \emph{contrast} $V_p(a)-(1-\eta)V_p(b)$ can be bounded directly. Writing $c=\sqrt{1-\eta}$ and $Z^{\pm}_{ij,ab,s}=D_{ij,a,s}\pm cD_{ij,b,s}$,
\begin{equation}
V_p(a)-(1-\eta)V_p(b)=\textstyle\sum_{i<j}p_ip_j\,\E Z^{-}_{ij,ab,s}\,\E Z^{+}_{ij,ab,s}.
\label{eq:contrast}
\end{equation}
\emph{Rectangular} bounds take Student intervals for all $2E(E-1)\binom{K}{2}$ contrast means and bound each product by its endpoint products. \emph{Covariance-aware} bounds instead place one Hotelling ellipse \citep{hotelling1931} around $(\Delta_{ij,a},\Delta_{ij,b})$ for each candidate pair and unordered probe pair, and bound $x_1^2-(1-\eta)x_2^2$ over it.

\begin{proposition}[Covariance-aware audit]\label{prop:cov}
Assume i.i.d.\ jointly Gaussian seed vectors per candidate pair and use level $\alpha/M_2$ per ellipse, where $M_2=\binom{K}{2}\binom{E}{2}$. Then the weighted sums of the ellipse extrema bound the contrast $V_p(a)-(1-\eta)V_p(b)$ for all $p$, ordered probe pairs and $\eta\in[0,1)$ simultaneously with probability at least $1-\alpha$. Support requires every rival lower bound to be positive and contradiction some rival upper bound to be negative (proof in App.~\ref{app:cov}).
\end{proposition}

Table~\ref{tab:families} summarizes the trade-offs. Both contrast families hold at one fixed pilot size and do not inherit the race's anytime guarantee; we report each with its own $\alpha=.05$ and do not select between them after seeing outcomes.

\begin{table}[!htb]
\centering
\begin{tabular}{lccc}
\toprule
 & Separate & Rect. & Hotelling \\
\midrule
Bounds & $V_p(e)$ & contrast & contrast \\
Family size$^\dagger$ & 90/look & 900 & 225 \\
Cross-probe covariance & no & partly & yes \\
Anytime (adaptive $n$) & yes & no & no \\
Uniform over $\eta$ & yes & no & yes \\
\bottomrule
\end{tabular}
\caption{Three audit families; all are uniform over $p$. $^\dagger$For $K{=}E{=}6$. ``Partly'': products of paired seed contrasts, with no joint region.}
\label{tab:families}
\end{table}

\section{Three-stage experimental design}

\paragraph{Stage I: equal-budget certification on stored libraries.}
Eight libraries of pre-collected training-time responses: 92,160 logistic-regression fits on Digits, Wine, Breast Cancer and Iris \citep{pedregosa2011}; 69,120 fits on noisier subsamples (Digits-200, Wine-80, Breast Cancer-120); and 6,048 small CNN fits on Digits. Each response is clean test log-loss after retraining a candidate under a probe. The six logistic candidates are reference, strong regularization, no scaling, 20\% label corruption, reduced data and PCA; the eight probes are retraining, half data, label noise, feature noise, weaker or stronger regularization, forced standardization and early stopping. The CNN uses six candidates and six probes, including learning-rate change and doubled epochs (Appendix~\ref{app:library}). Because these pools were collected before the present audit was designed, Stage I is a re-analysis: each library's calibration pool is split once at random into pilot and reference halves (200/200 seeds for logistic, 64/64 for CNN), and a separate target pool (80 or 40 seeds) serves as a second reference. For each library and prior (uniform or Dirichlet(1)) we draw 120 pilot permutations and posteriors. Uniform and focused allocation start from the same eight seeds per probe and spend exactly $KEn$ fits for $n\in\{8,16,32\}$; focused allocation adds a batch to the less-sampled member of the current leader--rival pair. We evaluate $\eta\in\{0,.05,.10\}$ and score normalized regret $R=(\max_e V_p(e)-V_p(a))/\max_e V_p(e)$ on both references.

\paragraph{Stage II: sequential diagnosis and stress tests.}
Each episode hides one of six candidates (50 episodes per candidate and library); a policy sees only the chosen probe's response on a fresh target seed, may repeat probes, and stops at posterior 0.95 or eight probes. A wrong mode crossing 0.95 is a confident error. Comparators are random, fixed-budget MDV, quadrature EIG, a variance-aware Schwaab-style score and paired or cellwise races; fixed MDV and EIG at 16 seeds share calibration fits. Simulated worlds vary shared-seed correlation and near ties; a drift test offsets calibration means and grades against unshifted targets (Appendix~\ref{app:seq}).

\paragraph{Stage III: external-policy trace audit (newly executed).}
A ninth library on scikit-learn Diabetes regression uses six Ridge configurations and six probes, scored by test MSE: 13,824 new fits in four disjoint 96-seed pools (pilot, references A and B, target). MDV, EIG, random and MinMDV (a negative control that minimizes plug-in MDV) each run 240 balanced four-step episodes with identical pilot and target tapes, and the audit observes their actions without altering them. The primary analysis uses $n=32$ pilots (1,152 fits per episode); $n=96$, the contrast audits at $\eta=.10$, the power decomposition and the error-spending ablation were added after the primary result and are exploratory (Appendix~\ref{app:stage3}).

\paragraph{Cost accounting.}
One fit per candidate--probe--seed response gives calibration cost $C_{\mathrm{cal}}=K\sum_e n_e$; a target probe costing $c$ fit units adds $cT$. Equal fit counts do not imply equal wall-clock cost, so we also time every logistic probe and reweight Stage I allocations by measured seconds (Appendix~\ref{app:wallclock}).

\section{Results: reliability is not validity}

\subsection{Stage I: equal-budget certification}
At 16 equivalent seeds, focused allocation certifies 647/1,920 decisions exactly (33.7\%) against 514 (26.8\%) for uniform allocation; at 32 seeds, 750 against 689, with 62 draws certified only under focusing and one only under uniform (Table~\ref{tab:stage1}, Fig.~\ref{fig:stage1}). A 10\% tolerance raises focused certification to 923 (48.1\%) against 732 (38.1\%). No issued certificate exceeds its tolerance on either reference. Repeating the whole audit with five more partition seeds gives exact gains of 3.2--6.5 points and 10\%-tolerance gains of 9.9--12.6 points, again without violations (Appendix~\ref{app:partition}). These are sensitivity ranges over shared pools, not confidence intervals.

\begin{table}[!htb]
\centering
\begin{tabular}{cr|rr|rr}
\toprule
 & & \multicolumn{2}{c|}{Exact (\%)} & \multicolumn{2}{c}{10\% tol. (\%)} \\
$n$ & Decisions & Uniform & Focused & Uniform & Focused \\
\midrule
8 & 1920 & \textbf{18.9} & \textbf{18.9} & \textbf{19.8} & \textbf{19.8} \\
16 & 1920 & 26.8 & \textbf{33.7} & 29.9 & \textbf{37.6} \\
32 & 1920 & 35.9 & \textbf{39.1} & 38.1 & \textbf{48.1} \\
\bottomrule
\end{tabular}
\caption{Stage I certification at identical fit counts; bold marks the better allocation. No issued certificate violates either held-out reference; all-score coverage on the first reference is 99.7--99.9\%.}
\label{tab:stage1}
\end{table}

The aggregate hides strong heterogeneity. At 32 seeds, a 10\% tolerance lifts Breast Cancer from 50 to 159 of 240 and Iris from 3 to 35, while Breast Cancer-120 stays at 1 and CNN-Digits already certifies 238 exactly. Tolerance helps when the best probes are close; it cannot create separation under severe noise. The references themselves disagree on the maximizer in 8.3\% of draws, and in 56.7\% on Digits-200, so exact-match accuracy against a single reference overstates how sharp the ground truth is. Meanwhile the focused plug-in choice has mean regret 0.91\% (uniform: 1.27\%): selection can be good when no certificate is available, so certificate rate and decision utility must be reported separately.

\begin{figure}[!t]
\centering
\includegraphics[width=\textwidth]{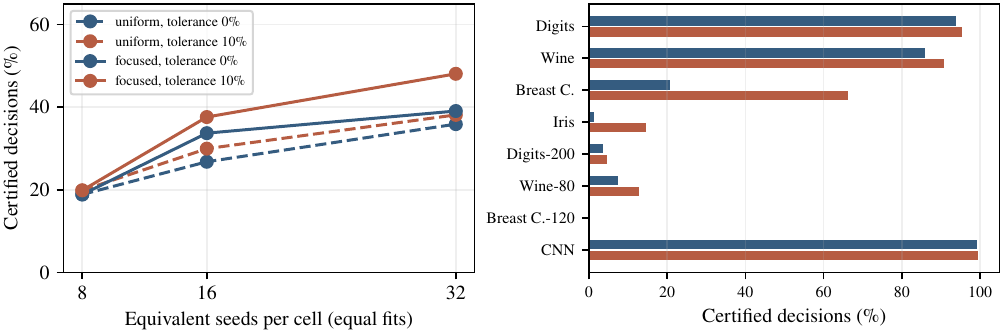}
\caption{Stage I on eight libraries. Left: certification at identical training-call budgets. Right: library heterogeneity at 32 equivalent seeds per cell, focused allocation. Tolerance is relative score regret, not a bound on wrong diagnoses.}
\label{fig:stage1}
\end{figure}

\paragraph{Wall-clock and error spending.}
Reweighting each Stage I decision by measured per-probe fit times changes the focused/uniform cost ratio to 1.007 at 32 seeds (0.97--1.03 per library), so the fit-count comparison is not an artifact of focusing on cheap probes, even though early stopping costs only 0.32--0.76 of a retraining. The error-spending schedule in \eqref{eq:radius} reserves $1/9$ of $\alpha$ for looks beyond the 64-seed cap. Renormalizing it over the eight reachable looks is free and adds 1--8 certificates per cell. Spending $\alpha/8$ equally per look instead trades early for late power: focused certification at 10\% tolerance falls from 381 to 342 at 8 seeds but rises from 923 to 981 at 32 (Appendix~\ref{app:spending}). The schedule should therefore follow the expected stopping time.

\subsection{Pairing, focusing and the unit of cost}
In simulated worlds with clearly separated probes, the focused paired race uses 26.0--67.2\% fewer responses than a cellwise race as correlation rises from $-0.2$ to $0.9$ (245 against 608 at $\rho=0.6$). On real responses with a uniform posterior, Wine needs 840 fits paired against 1,331 cellwise, and Breast Cancer certifies 77\% of the time paired against 0\% cellwise; Iris certifies under neither (this comparison grades against the pooled calibration mean). Pairing is not always cheaper end to end: when one fitted model can serve many probes at test time, uniform evaluation costs fewer fits. An evaluation must fix the unit of computation before comparing schedules.

\subsection{Stage II: calibration error reaches the diagnosis}
On the four standard libraries (1,200 episodes per policy, Table~\ref{tab:stage2}), MDV with four pilots per cell identifies 91.0\% of planted causes and errs confidently in 5.9\%; at 16 pilots, 95.8\% and 2.4\%. The first-step paired race reaches 96.0\% with 1.9\% confident errors, but spends 1,185 pilot fits on average against 768 for fixed 16-seed MDV, and its paired accuracy difference is $+0.25$ points (95\% bootstrap interval $[-0.58,1.08]$). EIG at matched calibration reaches 96.2\%. The race is a stopping rule, not a universally cheaper policy, and benchmarks should compare against strong fixed-budget baselines.

\begin{table}[!htb]
\centering
\begin{tabular}{lrrrr}
\toprule
Policy & Acc.\% & Conf.\ err.\% & Pilot & Target \\
\midrule
Random, $n=8$ & 91.1 & 4.0 & 384 & 3.88 \\
MDV, $n=4$ & 91.0 & 5.9 & 192 & 2.62 \\
MDV, $n=8$ & 94.5 & 3.7 & 384 & 2.60 \\
MDV, $n=16$ & 95.8 & 2.4 & 768 & 2.57 \\
EIG, $n=16$ & 96.2 & 2.2 & 768 & 2.52 \\
MDV, $n=32$ & 96.3 & 2.1 & 1536 & 2.56 \\
Pair race, first & 96.0 & 1.9 & 1185 & 2.61 \\
Pair race, every & \textbf{96.6} & \textbf{1.8} & 1890 & 2.57 \\
\bottomrule
\end{tabular}
\caption{Stage II diagnosis, pooled over four libraries; bold marks the best accuracy and fewest confident errors. Pilot and target are mean counts; target probes are capped at eight. MDV and EIG at $n=16$ share calibration fits; adaptive budgets are not matched.}
\label{tab:stage2}
\end{table}

\paragraph{Failures of the score and of the sampling model.}
On Digits-200 at 16 matched pilots, MDV identifies 84.3\% and EIG 92.3\%: under heteroscedastic noise a well-certified MDV maximizer can be the wrong probe for diagnosis. On the three noisy libraries, raising pilots from 16 to 64 brings no clear gain. On CNN-Digits, fixed MDV rises from 92.3\% to 95.7\% between 16 and 64 seeds, while the every-step race reaches 94.3\% at 43\% of the 64-seed budget (difference $-1.3$ points, $[-3.0,0.0]$), which does not establish non-inferiority. Under calibration drift (Fig.~\ref{fig:power}b), all-score coverage falls from 100\% to 56\% at offset SD 0.25 and to 7.5\% at 0.50, yet 170 of 200 worlds still certify, five of them wrongly. The theorem concerns the calibration distribution; certificates can outlive the validity of the measurements they rest on.

\subsection{Stage III: can an audit detect a poor action?}
Separate score intervals leave all 3,840 actions per budget unresolved. Cross-probe contrasts recover some power (Table~\ref{tab:stage3}). At 96 pilots the covariance-aware audit refutes 121 of 960 MinMDV actions and 23 random actions, against 109 and 22 with rectangles, and supports 83 random actions, which rectangles never do. No MDV or EIG action is supported or refuted, and no decisive label violates either reference. MDV has lower reference regret than EIG (0.6\% against 1.9\% at $n=32$) but slightly lower final accuracy: a reliable score ranking is not a diagnosis guarantee.

\begin{table}[!htb]
\centering
\begin{tabular}{lrrrrrr}
\toprule
 & \multicolumn{2}{c}{Ell.\ ref.} & \multicolumn{2}{c}{$n=96$} & Acc. & Reg. \\
Policy & 32 & 96 & sup. & Rect. & (\%) & (\%) \\
\midrule
MDV & 0 & 0 & 0 & 0 & 44.2 & \textbf{0.6} \\
EIG & 0 & 0 & 0 & 0 & \textbf{46.7} & 1.9 \\
Random & 1 & 23 & 83 & 22 & 39.6 & 24.2 \\
MinMDV & 10 & 121 & 0 & 109 & 39.6 & 37.3 \\
\bottomrule
\end{tabular}
\caption{Stage III traces: refuted (ref.) and supported (sup.) actions out of 960 per policy and budget, by the covariance-aware audit (Ell.) or rectangles (Rect., refutations); accuracy out of 240 episodes and mean regret against reference A, both at $n=32$ (accuracy at $n=96$: 45.8, 45.8, 38.8, 42.9); bold marks the best accuracy and lowest regret. No action is supported at $n=32$ by either contrast audit, rectangles support none at either budget, and separate score intervals label nothing. No decisive label violates either reference.}
\label{tab:stage3}
\end{table}

\paragraph{Why random actions are supported but MDV's never are.}
At $n=96$ every episode's pilot tape is the full 96-seed pool, so all 240 first-step states share one tape and the uniform prior: they are a single decision state replicated 240 times. In that state MDV always picks the plug-in maximizer, \emph{weaker ridge} ($\widehat{V}=7.82\times10^{-4}$). Two near-duplicates, \emph{retrain} (7.65) and \emph{force scaling} (7.67), lie within 10\% of it and differ from each other only through the unscaled candidate. All three must beat the binding rival, \emph{feature noise} (7.03), by the 10\% margin. The contrast interval of weaker ridge against it is 16\% wider (3.24 against $2.80\times10^{-4}$), plausibly because dividing $\lambda$ by ten mostly moves the strongly regularized candidate, and its lower bound is $-0.2\times10^{-5}$, while the other two have positive lower bounds. The audit therefore certifies the two low-variance near-ties and leaves the noisier maximizer unresolved. Random proposals hit these two probes 78 times at step 1 and five times at step 2, which accounts for all 83 supports. Repeated probes are allowed throughout, so this is not a sampling-without-replacement artifact. The lesson for evaluators is that the plug-in argmax is preferentially a high-variance probe, so the most certifiable action can differ from the policy's choice. An audit-aware MDV that executes a supported action when one exists would act with a certificate in this state at a reference-A regret of 2.4\%.

\begin{figure}[!t]
\centering
\includegraphics[width=\textwidth]{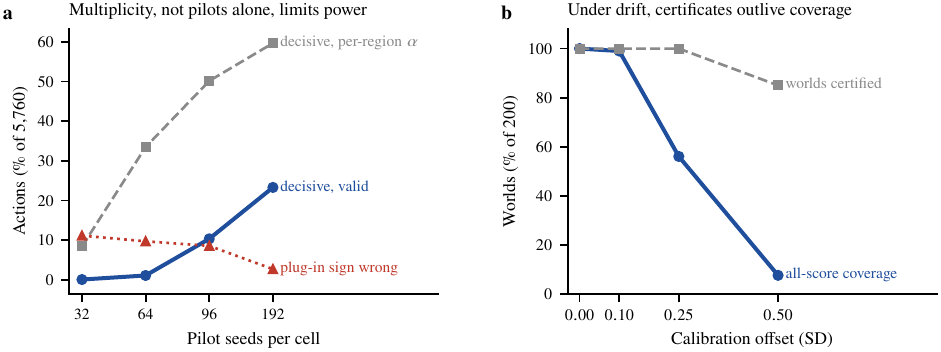}
\caption{(a) Stage III power decomposition on the 960 MDV decision states, replaying all six actions (5,760 labels per point). ``Decisive, valid'': supported or refuted by the covariance-aware audit at family error $\alpha/M_2$. ``Per-region $\alpha$'': the same ellipses without the multiplicity correction (diagnostic only, not a valid family). ``Plug-in sign wrong'': labels from point estimates alone that contradict a reference. At $n=192$ the pilot pool absorbs reference B and grading uses reference A only. No valid or per-region label violates a reference. (b) Stage II calibration drift over 200 simulated worlds per offset.}
\label{fig:power}
\end{figure}

\paragraph{Where the power goes.}
Replaying all six actions at the 960 recorded MDV states holds beliefs, tapes and references fixed (Fig.~\ref{fig:power}a, Appendix~\ref{app:cov}). At 96 pilots the covariance-aware audit labels 10.4\% of the 5,760 actions decisively (480 supported, 117 refuted) without violations; rectangles label 2.0\%, refutations only. The same ellipses without the $M_2=225$-fold correction would label 50.2\%, still without violations, while signs of the point estimates alone contradict a reference for 8.6\% of actions. Multiplicity, more than per-contrast sampling noise, is thus the dominant power loss. Doubling the pilots to 192 lifts valid decisive labels only to 23.3\%. Because the 96- and 192-seed points each use a single tape, the supported counts there reflect few distinct states. Across 800 known-truth Gaussian worlds both contrast families cover all pair contrasts in at least 794 worlds and produce no false action label (Appendix~\ref{app:cov}).

\section{Implications for evaluating AI scientists}\label{sec:implications}

\paragraph{A trace-level reporting protocol.}
A benchmark can instrument experiment selection by requiring the agent to declare its candidate library, interventions, utility, tolerance and budget before target feedback. The evaluator retains seed-level responses, separates pilot from target runs, reserves an independent reference, and logs for each step the belief, proposed probe, bounds, status and cost (Table~\ref{tab:report}). A trace with only the winning probe cannot show whether the choice was evidence-supported or lucky.

\begin{table}[!htb]
\centering
\begin{tabularx}{0.85\textwidth}{>{\raggedright\arraybackslash}p{0.32\linewidth}>{\raggedright\arraybackslash}X}
\toprule
Evaluation field & Interpretation \\
\midrule
Declared score and tolerance & Target of the guarantee; not automatically a valid scientific utility. \\
Certificate rate and violations & Share of decisions certified and observed errors against references, with denominators. \\
Reference disagreement & Sensitivity to finite reference estimation; mark ambiguous rankings. \\
Regret and diagnosis & Score loss, planted-cause accuracy, confident wrong conclusions. \\
Resource accounting & Pilot fits, target interventions, wall-clock, tokens where applicable. \\
Abstention and fallback & Separate verified, unresolved and forced actions. \\
\bottomrule
\end{tabularx}
\caption{Minimum report for an experiment-selection evaluation. The fields measure distinct constructs and should not be merged into one leaderboard number.}
\label{tab:report}
\end{table}

\paragraph{Auditing an LLM policy.}
The audit needs only the pilot tape, the policy's belief $p_t$ (or a uniform prior when none is declared) and its proposed probe. An LLM agent that reads candidate and probe descriptions, pilot summaries and past outcomes, and returns a probe, is therefore audited exactly like the scripted policies, and its tokens enter the resource row of Table~\ref{tab:report}. The released harness implements this interface for any chat-completion endpoint; we report no LLM results in this paper.

\paragraph{Three distinctions.}
The evidence supports three distinctions. \emph{Selective reliability is not accuracy}: certification can be rare while plug-in choices are good. \emph{Ranking reliability is not construct validity}: a well-estimated score can favor a probe that is poor for diagnosis. \emph{Finite-reference agreement is not population validity}: shared pools and uncertain references limit what violation counts mean. A policy that abstains more, or spends more to estimate its utility, can look better on a selectively reported accuracy metric.

\paragraph{Limitations.}
We use small models, planted defects and externally supplied finite libraries, and we do not validate the instrument against naturally occurring failures judged by domain experts. The Gaussian interval assumes stationary i.i.d.\ paired differences; the median variant changes the estimand; and drift breaks both. Stage I re-analyzes pre-collected pools with sampling without replacement, so its counts are finite-pool descriptions, and the Stage III extensions are exploratory. Even at 96 pilots most Stage III actions stay unresolved. Tighter multiplicity control (e.g., closed testing over contrasts), variance-aware utilities, cost-sensitive allocation and audits of LLM-generated libraries are open next steps.

\section{Conclusion}
ProbeAudit evaluates whether the next experiment chosen by an AI research policy is supported by reliable pilot evidence. Matched-seed calibration makes a classical disagreement ranking auditable; focused allocation and relative-regret tolerances raise certification at equal cost; covariance-aware contrasts recover part of the power that separate intervals lose, with multiplicity as the main remaining bottleneck. Diagnosis and drift experiments show that a reliable ranking is not a trustworthy discovery. We recommend that evaluations of AI scientists report decision reliability, reference uncertainty, abstention, resource cost and diagnostic validity as separate measurements.

\paragraph{Reproducibility and AI assistance.}
The anonymous artifact contains all response tables, generation and analysis scripts, outputs, the seed-split manifest, the revision analyses (\texttt{revision\_experiments.py}) and the LLM-policy harness. Generative AI tools assisted with analysis scripting, prototyping of extensions and manuscript editing; the authors independently verified all derivations, experimental outputs and claims.


\clearpage
\appendix
\section{Library construction and data provenance}\label{app:library}
All pools are measured responses; each training-time response is a separate fit. Standard logistic libraries use six candidates, eight probes, 400 calibration and 80 target seeds per dataset ($4\times6\times8\times480=92{,}160$ fits); the three reduced-data libraries share this layout (69,120 fits). CNN-Digits uses six candidates, six probes, 128 calibration and 40 target seeds (6,048 fits). The artifact also stores 48,720 test-time-probe fits used only for the cost remark in Section~5.2. Logistic libraries use calibration seeds 20000--20399 and target seeds 30000--30079; CNN uses 40000--40127 and 41000--41039. Seeds are disjoint between pilot, reference and target roles within a library. The Stage I split is stored in \texttt{workshop\_audit\_seed\_split.csv}. Candidate order is lexicographic, seeds are checked for alignment across candidates, interval radii are precomputed for $n=8,16,\dots,64$, both allocations are capped at 64 seeds per probe, and ties break to the first index. Budgets are nested on one tape, so decisions at different budgets and tolerances are not independent.

\begin{table}[H]
\centering
\begin{tabular}{lrrr}
\toprule
Library & Exact & 5\% & 10\% \\
\midrule
Digits & 225 & 227 & 229 \\
Wine & 206 & 212 & 218 \\
Breast Cancer & 50 & 87 & 159 \\
Iris & 3 & 17 & 35 \\
Digits-200 & 9 & 9 & 11 \\
Wine-80 & 18 & 24 & 31 \\
Breast C.-120 & 1 & 1 & 1 \\
CNN-Digits & 238 & 238 & 239 \\
\midrule
Total & 750 & 815 & 923 \\
\bottomrule
\end{tabular}
\caption{Focused allocation at 32 equivalent seeds per cell, 240 decisions per library.}
\label{tab:perlib}
\end{table}

\section{Partition sensitivity}\label{app:partition}
The primary RNG seed is 72611026; five further analyses use 72611027--72611031, each with new prior draws and pilot permutations on the same pools. The primary split was fixed before the others were run. All 27,331 certificates issued across the six analyses at 32 seeds have zero observed violations on either reference; overlapping decisions and shared pools preclude a binomial interpretation.

\begin{table}[H]
\centering
\begin{tabular}{rrrrr}
\toprule
 & \multicolumn{2}{c}{Exact} & \multicolumn{2}{c}{10\% tolerance} \\
Seed suffix & Uniform & Focused & Uniform & Focused \\
\midrule
26 & 689 & \textbf{750} & 732 & \textbf{923} \\
27 & 643 & \textbf{768} & 708 & \textbf{950} \\
28 & 646 & \textbf{735} & 708 & \textbf{913} \\
29 & 677 & \textbf{752} & 744 & \textbf{935} \\
30 & 626 & \textbf{731} & 688 & \textbf{882} \\
31 & 675 & \textbf{750} & 721 & \textbf{918} \\
\bottomrule
\end{tabular}
\caption{Certified decisions out of 1,920 at 32 equivalent seeds per cell; bold marks the better allocation.}
\label{tab:partition}
\end{table}

\section{Error-spending ablation}\label{app:spending}
Stage I inspects $n=8,16,\dots,64$, i.e.\ eight looks. \emph{Harmonic} is Eq.~\eqref{eq:radius}; \emph{truncated} multiplies its per-look level by $9/8$, which remains valid because the eight reachable looks sum to $8/9$; \emph{equal} spends $\alpha/(8M)$ per look. All three runs share tapes and posteriors with the primary analysis, and the harmonic run reproduces it exactly. No schedule produces a reference violation.

\begin{table}[H]
\centering
\begin{tabular}{lrrrrrrr}
\toprule
 & & \multicolumn{3}{c}{Exact} & \multicolumn{3}{c}{10\% tol.} \\
 & $n$ & H & T & E & H & T & E \\
\midrule
Focused & 8 & 362 & \textbf{366} & 325 & 381 & \textbf{382} & 342 \\
Focused & 16 & 647 & \textbf{652} & 619 & 722 & \textbf{725} & 705 \\
Focused & 32 & 750 & 751 & \textbf{789} & 923 & 929 & \textbf{981} \\
Uniform & 16 & 514 & \textbf{517} & 511 & 575 & \textbf{583} & 562 \\
Uniform & 32 & 689 & 691 & \textbf{702} & 732 & 733 & \textbf{752} \\
\bottomrule
\end{tabular}
\caption{Certified decisions out of 1,920 under harmonic (H), truncated (T) and equal (E) error spending (uniform at $n=8$ equals focused); bold marks the best schedule per row and tolerance.}
\label{tab:spending}
\end{table}

\section{Wall-clock accounting}\label{app:wallclock}
We timed every candidate--probe fit of the seven logistic libraries on four seeds (single CPU thread) and weighted each Stage I decision by the summed median seconds of the fits it consumed. Relative to retraining, early stopping costs 0.32--0.76, forced standardization 0.60--0.91 and label noise 0.99--1.41. The focused/uniform wall-clock ratio is 1.000, 1.007 and 1.007 at 8, 16 and 32 seeds, and 0.97--1.03 per library. The CNN library was not re-timed. Absolute times depend on hardware; only ratios are reported.

\section{Distribution-free certificates}\label{app:median}
For an i.i.d.\ sequence with median $\nu$, $\Pr(D\le\nu)\ge\frac12$ and $\Pr(D\ge\nu)\ge\frac12$. The event $D_{(j)}>\nu$ implies at most $j-1$ samples at or below $\nu$, and symmetrically for $D_{(n-j+1)}<\nu$; each is bounded by $\Pr\{\mathrm{Bin}(n,\frac12)\le j-1\}$ even with atoms. A union bound over tails, pairs, probes and looks reproduces the event of Theorem~\ref{thm:ranking} for $V_p^{\mathrm{med}}$. With $K=6$ and $E=8$, a finite interval requires at least 14 seeds at the first look, so this variant must start at $n_0\ge14$. In the stored race outputs (200 replications per cell, $n_0=16$, 64-seed cap), median certificates certify 52.4\% of standard-library races against 62.0\% for Student intervals at 10\% more fits (1,482 against 1,349), 6.5\% against 5.9\% on the noisy libraries, and 100\% against 99.3\% on CNN-Digits. No certificate of either family is wrong against the pooled reference, and median coverage is 100\% in every cell. The median variant is a robustness check, not a more powerful default.

\section{Sequential diagnosis specification}\label{app:seq}
{\sloppy At each probe let $\widehat m_{h,e}$ and $s_{h,e}^2$ be the pilot mean and variance from $n_e$ seeds, and $q_e=\max(\mathrm{median}_h s_{h,e}^2,\,0.003^2)$. The predictive variance is\par}
\[
\widehat v_{h,e}=\max\left\{\frac{(n_e-1)s_{h,e}^2+4q_e}{n_e+3}\left(1+\frac{1}{n_e}\right),\,0.004^2\right\},
\]
and the posterior update is
\[
p_{t+1}(h)\propto p_t(h)\,\widehat v_{h,e_t}^{-1/2}\exp\left[-\frac{(y_t-\widehat m_{h,e_t})^2}{2\widehat v_{h,e_t}}\right].
\]
EIG uses 48-node Gauss--Hermite quadrature for the predictive mixture. The Schwaab-style comparator divides each squared pairwise gap by the sum of the two predictive variances and the variance of the paired mean difference. These are fixed settings, not tuned or certified uncertainty estimates. All logistic libraries use stratified 72/28 splits, with a seed-dependent stratified subsample first for the reduced variants. The CNN has two convolutional and two fully connected layers, batch size 64, learning rate 0.002 and 12 epochs (24 under the doubled-epoch probe); AdamW replaces Adam for the heavy-weight-decay candidate. In the drift test, calibration means receive independent Gaussian offsets with the stated SD, while targets are drawn from the unshifted means.

\section{Stage III protocol}\label{app:stage3}
Diabetes is loaded without global scaling and the target divided by 100. Each seed fixes a 72/28 split, a training permutation and noise arrays shared by all candidates and probes; preprocessing is fit on the training subset. Candidates: reference Ridge ($\lambda=1$), strong Ridge ($\lambda=100$), unscaled features, training-label noise (SD 0.5), half training data, and PCA to three components. Probes: retrain, halve the training data, add label noise (SD 0.35), add feature noise (0.35 times the training-feature SD), divide $\lambda$ by ten, or force train-fitted scaling. Half-data perturbations compose. Seeds 51000--51095, 52000--52095, 53000--53095 and 54000--54095 form the pilot, reference A, reference B and target pools; they are disjoint seed sets over one public dataset, not disjoint patient cohorts.

Using RNG seed 7261126, each of 40 episodes per hidden candidate samples a pilot tape without replacement from the pilot pool (all 96 seeds when $n=96$) and four target seeds. Policies start from a uniform prior, update as in Appendix~\ref{app:seq} with variance floor $10^{-8}$, may repeat probes, and execute all four steps; accuracy uses the terminal posterior. Confident wrong events at $n=32$ number 3, 3, 3 and 4 for MDV, EIG, random and MinMDV; at $n=96$, one each. Separate-score intervals use Bonferroni level $\alpha/[E\binom{K}{2}]$ per pair--probe mean. We make no normality claim for these MSE differences and treat this stage as a finite-pool stress test.

\section{Covariance-aware audit: proof and validation}\label{app:cov}
Fix $n>2$ matched seeds. For a candidate pair $i<j$ and probes $a<b$, let $X_s=(D_{ij,a,s},D_{ij,b,s})^\top$ have mean $\mu$, sample mean $\hat\mu$ and unbiased covariance $S$, and set
\[
q=\frac{2(n-1)}{n(n-2)}F^{-1}_{2,n-2}(1-\alpha/M_2),\qquad M_2=\tbinom{K}{2}\tbinom{E}{2}.
\]
The region is $C=\{\hat\mu+\sqrt{q}\,S^{1/2}u:\|u\|_2\le1\}$; for full-rank Gaussian $X_s$, $\Pr(\mu\in C)=1-\alpha/M_2$. If the true covariance has rank $r<2$, samples and mean lie in its affine support and the rank-$r$ pivot $(n-1)\chi^2_r/\chi^2_{n-r}$ is stochastically smaller than the rank-two one, so $C$ stays conservative; rank zero is deterministic. Numerically we clip negative eigenvalues and add $10^{-14}$ to the diagonal, which enlarges $C$.

\paragraph{Proof of Proposition~\ref{prop:cov}.}
A union bound over the $M_2$ regions gives joint coverage with probability $1-\alpha$. On this event, each quadratic $x^\top Ax$ with $A=\mathrm{diag}(1,-(1-\eta))$ (reversed for the order $(b,a)$) evaluated at the true mean lies between its extrema over $C$, for every $\eta$. Weighting by $p_ip_j\ge0$ and summing bounds the contrast for every $p$. Support and contradiction then imply the stated regret inequalities. No guarantee across calibration sizes or across a data-selected choice of family is claimed. \hfill$\square$

\paragraph{Numerical outer bounds.}
With $R=S^{1/2}$, $B=RAR$ and $g=2RA\hat\mu$, the extrema lie within $\hat\mu^\top A\hat\mu\mp\sqrt{q}\|g\|_2$ plus $q\min(0,\lambda_{\min}(B))$ or $q\max(0,\lambda_{\max}(B))$. Since the quadratic is indefinite for $\eta<1$, extrema lie on the boundary; we evaluate $G=128$ boundary angles and widen the grid extrema by $C_2(\pi/G)^2/2$, $C_2=\sqrt{q}\|g\|_2+4q\|B\|_{\mathrm{op}}$, then intersect with the analytic bounds (double precision, not verified interval arithmetic).

\begin{table}[H]
\centering
\begin{tabular}{rlrrrr}
\toprule
$n$ & Method & Support & Refute & Both bad & Both good \\
\midrule
32 & Rect. & 0 & 9 & 2763 & 2896 \\
32 & Ell. & 0 & 9 & 2763 & 2896 \\
96 & Rect. & 0 & 116 & 2768 & 2934 \\
96 & Ell. & 480 & 117 & 2768 & 2934 \\
\bottomrule
\end{tabular}
\caption{Common-state replay of all six actions at the 960 MDV states of each budget's own traces (5,760 labels). Both bad/good: both references agree on the 10\% regret label. All 480 supports come from the single replicated first-step state at $n=96$ (\emph{retrain} and \emph{force scaling}, 240 each). No decisive label violates a reference.}
\label{tab:replay}
\end{table}

\begin{table}[H]
\centering
\begin{tabular}{rlrrrr}
\toprule
$n$ & Variant & Support & Refute & Viol. & Refuted/bad \\
\midrule
32 & Valid & 0 & 6 & 0 & 6/2768 \\
32 & Per-region & 310 & 176 & 0 & 176/2768 \\
64 & Valid & 8 & 57 & 0 & 57/2768 \\
64 & Per-region & 1169 & 760 & 0 & 760/2768 \\
96 & Valid & 480 & 117 & 0 & 117/2768 \\
96 & Per-region & 1498 & 1391 & 0 & 1391/2768 \\
192 & Valid & 844 & 499 & 0 & 499/2799 \\
192 & Per-region & 1710 & 1727 & 0 & 1727/2799 \\
\bottomrule
\end{tabular}
\caption{Power decomposition (Fig.~\ref{fig:power}a) on the 960 MDV states recorded at $n=96$; 5,760 labels per row. For $n=32,64$ each episode receives a fresh random subset of the pilot pool; $n=192$ pools pilot and reference B and is graded on reference A only. Plug-in signs contradict a reference for 643, 561, 495 and 157 labels at the four budgets.}
\label{tab:powerdecomp}
\end{table}

\paragraph{Known-truth simulation.}
For each $(n,\rho)\in\{32,96\}\times\{0,0.8\}$ we draw 200 worlds with $K=E=6$ and means from $N(0,0.35^2)$; even worlds triple the first probe's means and odd worlds set the second probe's means to the first plus $N(0,0.02^2)$ noise (a near tie). Responses are $m_{h,e}+0.35[\sqrt{\rho}A_{h,s}+\sqrt{1-\rho}B_{h,e,s}]+0.15C_{e,s}$ with standard-normal arrays, uniform posterior, $\eta=0.10$, $\alpha=0.05$ per method and budget, and RNG seed 7261127. The ellipse supports and refutes more actions in every cell with no false label; the counts do not prove the theorem or guarantee power elsewhere.

\begin{table}[H]
\centering
\begin{tabular}{rrlrrrr}
\toprule
$n$ & $\rho$ & Method & Cover & Support & Refute & False \\
\midrule
32 & 0.0 & Rect. & 200 & 60 & 345 & 0 \\
32 & 0.0 & Ell. & 198 & \textbf{66} & \textbf{392} & 0 \\
32 & 0.8 & Rect. & 198 & 66 & 372 & 0 \\
32 & 0.8 & Ell. & 198 & \textbf{76} & \textbf{411} & 0 \\
96 & 0.0 & Rect. & 199 & 81 & 509 & 0 \\
96 & 0.0 & Ell. & 199 & \textbf{93} & \textbf{574} & 0 \\
96 & 0.8 & Rect. & 199 & 93 & 565 & 0 \\
96 & 0.8 & Ell. & 199 & \textbf{99} & \textbf{617} & 0 \\
\bottomrule
\end{tabular}
\caption{Known-truth Gaussian simulation: 200 worlds and 1,200 actions per row. Cover: worlds whose pair-contrast bounds all cover the truth. False: worlds with any false action label. Bold: more supports or refutations within each setting.}
\label{tab:sim}
\end{table}


\begin{thebibliography}{17}

\bibitem[{Box and Hill}(1967)]{box1967}
Box, G.~E.~P.; and Hill, W.~J. 1967.
\newblock Discrimination among mechanistic models.
\newblock \emph{Technometrics}, 9(1): 57--71.

\bibitem[{Chaloner and Verdinelli}(1995)]{chaloner1995}
Chaloner, K.; and Verdinelli, I. 1995.
\newblock Bayesian experimental design: A review.
\newblock \emph{Statistical Science}, 10(3): 273--304.

\bibitem[{Chan et~al.}(2025)]{chan2025}
Chan, J.~S.; Chowdhury, N.; Jaffe, O.; Aung, J.; Sherburn, D.; Mays, E.; Starace, G.; Liu, K.; Maksin, L.; Patwardhan, T.; Weng, L.; and M\k{a}dry, A. 2025.
\newblock MLE-bench: Evaluating machine learning agents on machine learning engineering.
\newblock In \emph{International Conference on Learning Representations}.

\bibitem[{Chen et~al.}(2025)]{chen2025}
Chen, Z.; Chen, S.; Ning, Y.; Zhang, Q.; Wang, B.; Yu, B.; Li, Y.; Liao, Z.; Wei, C.; Lu, Z.; Dey, V.; Xue, M.; Baker, F.~N.; Burns, B.; Adu-Ampratwum, D.; Huang, X.; Ning, X.; Gao, S.; Su, Y.; and Sun, H. 2025.
\newblock ScienceAgentBench: Toward rigorous assessment of language agents for data-driven scientific discovery.
\newblock In \emph{International Conference on Learning Representations}.

\bibitem[{Garivier and Kaufmann}(2016)]{garivier2016}
Garivier, A.; and Kaufmann, E. 2016.
\newblock Optimal best arm identification with fixed confidence.
\newblock In \emph{Conference on Learning Theory}, volume~49 of \emph{Proceedings of Machine Learning Research}, 998--1027.

\bibitem[{Glasserman and Yao}(1992)]{glasserman1992}
Glasserman, P.; and Yao, D.~D. 1992.
\newblock Some guidelines and guarantees for common random numbers.
\newblock \emph{Management Science}, 38(6): 884--908.

\bibitem[{Guo, Shamai, and Verd{\'u}}(2005)]{guo2005}
Guo, D.; Shamai, S.; and Verd{\'u}, S. 2005.
\newblock Mutual information and minimum mean-square error in Gaussian channels.
\newblock \emph{IEEE Transactions on Information Theory}, 51(4): 1261--1282.

\bibitem[{Hotelling}(1931)]{hotelling1931}
Hotelling, H. 1931.
\newblock The Generalization of Student's Ratio.
\newblock \emph{Annals of Mathematical Statistics}, 2(3): 360--378.

\bibitem[{Howard et~al.}(2021)]{howard2021}
Howard, S.~R.; Ramdas, A.; McAuliffe, J.; and Sekhon, J. 2021.
\newblock Time-uniform, nonparametric, nonasymptotic confidence sequences.
\newblock \emph{Annals of Statistics}, 49(2): 1055--1080.

\bibitem[{Huang et~al.}(2024)]{huang2024}
Huang, Q.; Vora, J.; Liang, P.; and Leskovec, J. 2024.
\newblock MLAgentBench: Evaluating language agents on machine learning experimentation.
\newblock In \emph{Proceedings of the 41st International Conference on Machine Learning}, volume 235 of \emph{Proceedings of Machine Learning Research}, 20271--20309.

\bibitem[{Hunter and Reiner}(1965)]{hunter1965}
Hunter, W.~G.; and Reiner, A.~M. 1965.
\newblock Designs for discriminating between two rival models.
\newblock \emph{Technometrics}, 7(3): 307--323.

\bibitem[{Kalyanakrishnan et~al.}(2012)]{kalyanakrishnan2012}
Kalyanakrishnan, S.; Tewari, A.; Auer, P.; and Stone, P. 2012.
\newblock PAC subset selection in stochastic multi-armed bandits.
\newblock In \emph{Proceedings of the 29th International Conference on Machine Learning}, 655--662.

\bibitem[{Kim and Nelson}(2001)]{kim2001}
Kim, S.-H.; and Nelson, B.~L. 2001.
\newblock A fully sequential procedure for indifference-zone selection in simulation.
\newblock \emph{ACM Transactions on Modeling and Computer Simulation}, 11(3): 251--273.

\bibitem[{Lindley}(1956)]{lindley1956}
Lindley, D.~V. 1956.
\newblock On a measure of the information provided by an experiment.
\newblock \emph{Annals of Mathematical Statistics}, 27(4): 986--1005.

\bibitem[{Pedregosa et~al.}(2011)]{pedregosa2011}
Pedregosa, F.; Varoquaux, G.; Gramfort, A.; Michel, V.; Thirion, B.; Grisel, O.; Blondel, M.; Prettenhofer, P.; Weiss, R.; Dubourg, V.; Vanderplas, J.; Passos, A.; Cournapeau, D.; Brucher, M.; Perrot, M.; and Duchesnay, E. 2011.
\newblock Scikit-learn: Machine learning in Python.
\newblock \emph{Journal of Machine Learning Research}, 12: 2825--2830.

\bibitem[{Rainforth et~al.}(2024)]{rainforth2024}
Rainforth, T.; Foster, A.; Ivanova, D.~R.; and Bickford~Smith, F. 2024.
\newblock Modern Bayesian experimental design.
\newblock \emph{Statistical Science}, 39(1): 100--114.

\bibitem[{Schwaab, Monteiro, and Pinto}(2008)]{schwaab2008}
Schwaab, M.; Monteiro, J.~L.; and Pinto, J.~C. 2008.
\newblock Sequential experimental design for model discrimination: Taking into account the posterior covariance matrix of differences between model predictions.
\newblock \emph{Chemical Engineering Science}, 63(9): 2408--2419.

\end{thebibliography}
\end{document}